\RequirePackage{fix-cm}
\documentclass[smallextended,glov3]{svjour3}       
\smartqed  
\usepackage{lineno,hyperref}
\usepackage{color}
\usepackage{graphicx}
\usepackage{mathtext}
\usepackage{amssymb,amsmath,latexsym,amsfonts,array,epsfig}
\usepackage{mathrsfs}
\usepackage [latin1]{inputenc}

\journalname{Neural Computing and Applications. 2021.\\ Vol. 33 (16). P. 10189-10198.}

\begin{document}

\title{Multi-Valued Neural Networks I}
\subtitle{A Multi-Valued Associative Memory}

\author{Dmitry Maximov \and Vladimir I. Goncharenko \and Yury S. Legovich}

\institute{Dmitry Maximov \at
              Trapeznikov Institute of Control Science Russian Academy of Sciences,\\
              65 Profsoyuznaya str, Moscow \\
              Tel.: +7-909-913-81-55\\
              \email{jhanjaa@ipu.ru, dmmax@inbox.ru}            \\
\and
           Vladimir I. Goncharenko \at
              Trapeznikov Institute of Control Science Russian Academy of Sciences,\\
              65 Profsoyuznaya str, Moscow \\
              Tel.: +7-495-334-87-21 \\
              \email{vladimirgonch@mail.ru} \\
              \and
              Yury S. Legovich \at
              Trapeznikov Institute of Control Science Russian Academy of Sciences,\\
              65 Profsoyuznaya str, Moscow \\
              Tel.: +7-495-334-87-21 \\
              \email{legov@ipu.ru}
}

\date{Received: date / Accepted: date}

\maketitle

\begin{abstract}
A new concept of a multi-valued associative memory is introduced, generalizing a similar one in fuzzy neural networks. We expand the results on fuzzy associative memory with thresholds, to the case of a multi-valued one: we introduce the novel concept of such a network without numbers, investigate its properties, and give a learning algorithm in the multi-valued case. We discovered conditions under which it is possible to store given pairs of network variable patterns in such a multi-valued associative memory. In the multi-valued neural network, all variables are not numbers, but elements or subsets of a lattice, i.e., they are all only partially-ordered. Lattice operations are used to build the network output by inputs. In this paper, the lattice is assumed to be Brouwer and determines the implication used, together with other lattice operations, to determine the neural network output. We gave the example of the network use to classify aircraft/spacecraft trajectories.


\end{abstract}

\section{Introduction}
\label{intro}
Feedforward fuzzy neural networks in which the internal operations
based on fuzzy operations of joins and meets $\vee - \wedge$, were proposed in \cite{kosko}.
Such structures are called fuzzy associative memory, and fuzzy information in them is
represented by the elements of the numeric interval [0, 1]. These networks are used for storing and classifying of fuzzy patterns. More complicated variants of fuzzy networks use fuzzy numbers as weights and variables. The most significant use of such networks lies in the field of approximation analysis and image restoration \cite{epdf}. Also, there is a self-organizing fuzzy network which reorganize the model and adapt itself to a changing environment \cite{br}.

However, often, it is preferable to operate not with number representation of information, but with linguistic expressions directly. Such a situation occurs, e.g., in decision making systems \cite{liter} and cognitive maps.
For example, when fuzzy assessing typical flight situations, a set of linguistic variables and their degrees of significance/certainty, estimate the emerging state. Here, confidence levels take, in turn, values in a numerical interval, usually [0,1]. However, a problem arises here: the methods for determining fuzziness are subjective --- experts assess the degree of fuzziness that implies some uncertainty.

In \cite{maximov19} it was demonstrated that it is not necessary to use numbers in such a situation: linguistic variables (but already partially-ordered, as it mentioned in \cite{liter}), which do not require a mandatory numerical evaluation, can serve again as assessments. The situation itself determines such estimations that do not need an expert opinion or, at least, the expert's evaluation of the situation is greatly facilitated.

In this case, to compare the valuations and obtain control solutions in \cite{maximov19}, the concepts of not fuzzy logic are used, but multi-valued, in which the scale of truth values is a Brouwer lattice of a general form. Such scales of truth values generalize a linearly ordered scale (in particular, in fuzzy logic) \cite{maximov16} and naturally define implication as a lattice operation.
Different variants of such multi-valued or lattice-valued logics\footnote{We use the term ``multi-valued'' for the special case of finite and non-linearly-ordered scale of truth values used in these logics.} can be found in \cite{novak}, \cite{latvallog}, \cite{manuscript}, \cite{Maximov_Ax}.

Such decision making systems can use neural networks in their processing. However, nobody has used general non-numeric multi-valued lattices for neural networks weights and data representation so far. Though, lattice-valued neural networks are known \cite{latval} as fuzzy ones generalization, variables in these networks are represented by set-valued functions with the sets are again \emph{numeric} subintervals of an interval in $\mathbb{R}$.

Thus, the main contribution of the paper is the first example of a general associative memory in which weights and data take values in a non-numeric multi-valued linguistic lattice and an application of such a network as a simple linguistic patterns classifier. However, the lattice does not have to be a linguistic variable one.
Such a lattice may be a lattice of sets or, even, a lattice of graphs of system state configurations as in the researches on a system state estimating, not by numbers, but by elements of a lattice: \cite{Maximov_17}, \cite{assa},\\ \cite{Maximov_R}, \cite{Maximov_an}.

We will use elements of a complete lattice in a multi-valued associative memory instead of elements of the interval [0, 1]. In this paper we suppose the lattice is Brouwer. The case of a residuated lattice is considered in \cite{Maximov_neuro2}.
The associative memory generalizes a fuzzy one of\\ \cite{epdf} to the case of using multi-valued logic operations (instead of fuzzy ones) with the multi-valued lattice elements.

In \cite{maximov20}, \cite{maximov20e} the \emph{simplest} case of a multi-valued associative memory without thresholds has been considered. In this paper, we expand these results (as well as \cite{epdf} on fuzzy associative memory with thresholds), to the case of a multi-valued associative memory with thresholds: we define the concept of such a network, investigate its properties, and give a generalization of the learning algorithm of fuzzy associative memory for a multi-valued case. Though, the inputs, outputs and connection weights of the network are linguistic variables, not numbers, this makes it possible to use such a network for processing control and diagnostic information of complex dynamic objects. We demonstrate this application by such a network use to classify aircraft/spacecraft trajectories.

The paper is organized as follows: in Sec. \ref{sec1} we give a brief list of definitions used in the text. In Sec. \ref{sec2} we define the multi-valued associative memory and prove some existence theorems for the solution of the equation defining the memory. In Sec. \ref{sec3} we consider a learning algorithm for the memory. In Sec. \ref{sec4} we discuss the previous results and the network computational complexity. In Sec. \ref{sec5} we give an example of such a network use as a linguistic pattern classifier. In Sec. \ref{sec6} we conclude the paper.

\section{Definitions}
\label{sec1}
\begin{definition}
A \textbf{lattice} is a partially-ordered set having, for any two elements, their exact upper bound or join $\vee$ (sup, max) and the exact lower bound or meet $\wedge$ (inf, min).
\end{definition}
\begin{definition}
The \textbf{exact upper bound} of the two elements is the smallest lattice element,
larger than both of these elements.
\end{definition}
\begin{definition}
The \textbf{exact lower bound} is dually defined as the largest element of the lattice, smaller than both the elements.
\end{definition}
\begin{definition}
A \textbf{complete lattice} is a lattice in which any two subsets have a join and a meet.
It follows from the definition that in a non-empty complete lattice, there is the biggest ``1'' and the smallest ``0'' elements.
\end{definition}
\begin{definition}
\textbf{Generators} of the lattice are called its elements, from which all the others are obtained
by join and meet operations.
\end{definition}
\begin{definition}
A lattice is called \textbf{atomic} if every two of its generators have null meets.
\end{definition}
If we take such a lattice as a scale of truth values in a multi-valued logic,
then the biggest element will correspond to complete truth (true), the smallest to complete falsehood (false), and intermediate elements will correspond to partial truth in the same way as  the elements of the segment [0,1] evaluate partial truth in fuzzy logic.

In logics, with such a scale of truth values, the implication can be determined by multiplying lattice elements, or internally, only from lattice operations.
\begin{definition}
A \textbf{Brouwer lattice} is a lattice that has internal implications.
\end{definition}
\begin{definition}\label{imp}
In such a lattice, the \textbf{implication} $c = a\Rightarrow b$ is defined as the largest $c:\;a\wedge b = a\wedge c$.
\end{definition}
\begin{definition}
The implication $\neg a = a\Rightarrow 0$ is called the \textbf{pseudo-complement} of $a$.
\end{definition}
Distribution laws for join and meet are satisfied in Brouwer lattices. The converse is valid only for finite lattices.

We will assume that the lattices used in the multi-valued neural network are complete,
distributive and finite.

\section{Multi-Valued $\vee - \wedge$ Associative Memory with Threshold}
\label{sec2}
As in the case of fuzzy associative memory \cite{epdf}, in $\vee - \wedge$ multi-valued associative memory, the transfer function normalizing outputs to the prefix range is not needed, since the $\vee - \wedge$ operators limit the outputs to the range of inputs, and $\wedge$ is also a threshold operator \\ \cite{blanco}. A $\delta$--function which activates the desired neuron in a one-layer multi-valued net computing a lattice implication, is considered in \cite{maximov20e} (see also Sec \ref{sec4}).

More complicated fuzzy neural networks, which work not with elements of $[0, 1]$, but with fuzzy numbers, and in need also of activation, use different transfer functions, from simplest to a uniform Tauber-Wiener function \cite{epdf}. However, a theory of a multi-valued analogue of fuzzy numbers might be developed in order to build complicated multi-valued neural networks with transfer functions. The idea of SPOCU function ``picking up'' the appropriate properties for activation function directly from training \cite{spocu} may be useful here (see also, Maximov D. ``Multi-Valued Cognitive Maps'').

Let us suppose, that an input signal is $x \in L^{n}$, and an output signal is $y \in L^{m}$, where $L$ is the lattice used. In this case, the input-output relationship in a two-layer multi-valued associative memory can be written as
$\textbf{y} = \textbf{x}\circ \textbf{W}$, where $\circ$ stands for the $\vee - \wedge$ composition operation and $\textbf{W} = (w_{ij})_{n\times m}\in\mu_{n\times m}$ is the $n\times m$ matrix of the weights of connections, with elements from a complete,
distributive, and finite lattice $L$, $i\in N = \{1 ... n\},\;j\in M = \{1 ... m\}$ (Fig. \ref{fig1}).
\begin{figure}\begin{center}
  \includegraphics[scale=0.7]{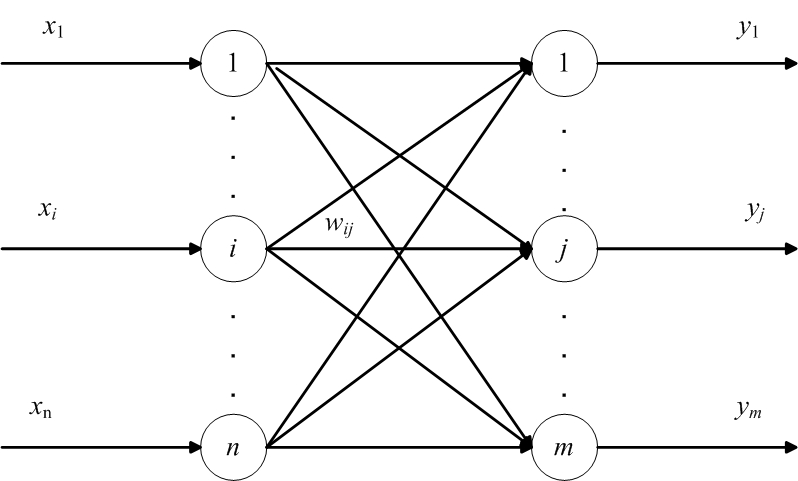}
\end{center}
\caption{Bilayer associative memory}
\label{fig1}       
\end{figure}
In \cite{maximov20e}, several simple composition options are suggested, one of them is:
\begin{equation}\label{eq1}
y_{j} = \bigvee_{i}\{x_{i}\wedge w_{ij}\}.
\end{equation}

In the theory of the fuzzy associative memory \cite{epdf}, exactly (\ref{eq1}) is considered, however, with fuzzy values and operations. There, the fuzzy operator $\vee$ generalizes the sum $\sum$ in ordinary neural networks, and the fuzzy operator $\wedge$ is used as a special case of multiplication. In multi-valued neural networks we use in (\ref{eq1}) the lattice operations $\vee$ and $\wedge$ with the same functionality, and $\wedge$ is again a special case of multiplication in residuated lattices.
More sophisticated fuzzy models use a combination of $\vee$ and $t$-norms \cite{epdf},  \cite{sussner} and a combination with implication \cite{sussner}. However, because we use not a linear-ordered segment as a set of variables, but the general Brouwer lattice without multiplying elements (in this paper), we do not consider these models.

Thus, we consider the composition \ref{eq1}, but with thresholds $c_{ij}$ and $d_{j}$ to input unit $i$ and output unit $j$ respectively (Fig. \ref{fig1}) with all variables taking values in the lattice $L$:
\begin{equation}\label{eq3}
y_{j} = (\bigvee_{i}\{(x_{i}\vee c_{ij})\wedge w_{ij}\})\vee d_{j} = \bigvee_{i}\{(x_{i}\vee c_{ij}\vee d_{j})\wedge (w_{ij}\vee d_{j})\}.
\end{equation}
In the vector notation, it can be written as:
\begin{equation}\label{eq4}
\textbf{y} = ((\textbf{x}\vee \textbf{c})\circ \textbf{W})\vee \textbf{d}.
\end{equation}
Here, all quantities are the elements (not subsets!) of the lattice $L$. Hence, no membership functions are needed unlike the fuzzy case. As with the activation function, the theory of a multi-valued analogue of fuzzy numbers should be developed in order to consider membership functions. Then, such multi-valued numbers may be used also in, e.g., quality modelling of a non-linear process similar to \cite{ar} or membership function construction as in \cite{cr}.

We denote $(X, Y) = \{\textbf{x}^{k}, \textbf{y}^{k}\;|\;k\in P\},\;P =\{1 ... p\}$ --- the family of pairs of multi-valued patterns with $\textbf{x}^{k} = (x_{1}^{k}, ..., x_{n}^{k})$ and $\textbf{y}^{k} = (y_{1}^{k}, ..., y_{m}^{k})$.
We introduce also the following sets (Fig. \ref{fig2}):

\begin{multline}\label{g}
G_{ij}(X,Y) = \{k\in P\;|\;x_{i}^{k} > y_{j}^{k}\},\; E_{ij}(X,Y) = \{k\in P\;|\;x_{i}^{k} = y_{j}^{k}\},\;\\ GE_{ij}(X,Y) = G_{ij}(X,Y)\cup E_{ij}(X,Y),\; L_{ij}(X,Y) = \{k\in P\;|\;x_{i}^{k} < y_{j}^{k}\},\;\\ LE_{ij}(X,Y) = L_{ij}(X,Y)\cup E_{ij}(X,Y),\; \\NC_{ij}(X,Y) = \{k\in P\;|\;x_{i}^{k}\not\in GE_{ij}(X,Y)\cup LE_{ij}(X,Y)\}.
\end{multline}
\begin{figure}\begin{center}
  \includegraphics[scale=1.0]{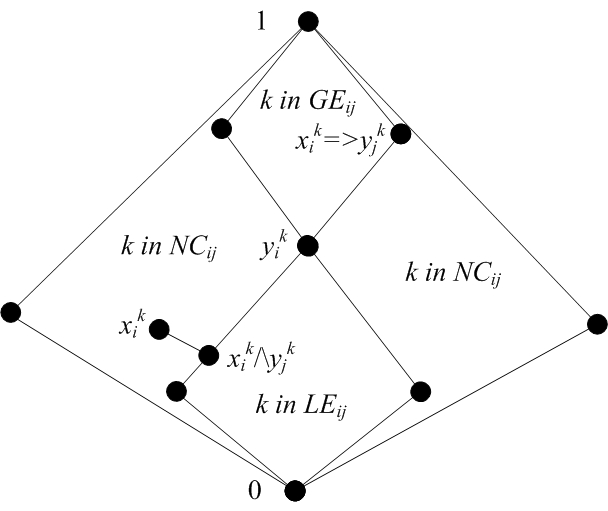}
\end{center}
\caption{Implication and areas of comparability}
\label{fig2}       
\end{figure}
The equalities define those sets of pattern pairs, for which $x$'s are grater than $y$'s (G), or less than y's (L), ... or not comparable with $y$'s (NC).

We also establish the threshold matrix
$\textbf{c}^{0} = (c^{0}_{11}, ..., c^{0}_{nm})$, and vector $\textbf{d}^{0} = (d^{0}_{1}, ..., d^{0}_{m})$, and the connection weight matrix $W_{0} = (w_{ij}^{0})_{n\times m}$ like in\\ \cite{maximov20} and in a different way from \cite{epdf}:
\begin{equation}\label{eqd0}
d_{j}^{0} = \bigwedge_{k\in P}(y_{j}^{k}),
\end{equation}
\begin{equation}\label{eqc0}c_{ij}^{0} = \left\{\begin{aligned}\bigwedge_{\substack{k\in LE_{ij}(X,Y),\\ j\in M}}(y_{j}^{k}),\;LE_{ij}(X,Y) \neq \emptyset; \\
                                    \bigwedge_{\substack{k\in P,\\ j\in M}}(y_{j}^{k}),\; LE_{ij}(X,Y) = \emptyset,
                                    \end{aligned}
                              \right.
\end{equation}
\begin{equation}\label{eq5}
w_{ij}^{0} = \bigwedge_{k\in P}(x_{i}^{k}\Rightarrow y_{j}^{k}).
\end{equation}
The input thresholds $c^{0}$'s are individual for each input-output connection. One should add an additional meet by $i$ in (\ref{eqc0}) to obtain the unique threshold for all inputs. This is the case of $LE_{ij}(X,Y) = \emptyset$ in (\ref{eqc0}). These constants  $c^{0}$'s raise to the level of outputs' intersections the inputs in patterns where $x$'s less than or equal to $y$'s. Thus, the capacity of the memory increases, since we can store more patterns with small $x$'s (Theorem \ref{th3}). Also, $d^{0}$'s adapts everything to the memory outputs.

Now, we define the sets ($i\in N,\;j\in M$):
\begin{multline}\label{Tg}
TG_{ij}((X,\textbf{c}^{0});\textbf{d}^{0},Y) = \{k\in P\;|\;x_{i}^{k}\vee c_{ij}^{0}\vee d_{j}^{0} > y_{j}^{k}\},\\
TE_{ij}((X,\textbf{c}^{0});\textbf{d}^{0},Y) = \{k\in P\;|\;x_{i}^{k}\vee c_{ij}^{0}\vee d_{j}^{0} = y_{j}^{k}\},\\
TL_{ij}((X,\textbf{c}^{0});\textbf{d}^{0},Y) = \{k\in P\;|\;x_{i}^{k}\vee c_{ij}^{0}\vee d_{j}^{0} < y_{j}^{k}\},\\
TGE_{ij}((X,\textbf{c}^{0});\textbf{d}^{0},Y) = TG_{ij}((X,\textbf{c}^{0});\textbf{d}^{0},Y)\cup TE_{ij}((X,\textbf{c}^{0});\textbf{d}^{0},Y)\}, \\
TLE_{ij}((X,\textbf{c}^{0});\textbf{d}^{0},Y) = TL_{ij}((X,\textbf{c}^{0});\textbf{d}^{0},Y)\cup TE_{ij}((X,\textbf{c}^{0});\textbf{d}^{0},Y)\},\\
TNC_{ij}((X,\textbf{c}^{0});\textbf{d}^{0},Y) = \{k\in P\;|\;k\not\in TGE_{ij}((X,\textbf{c}^{0});\textbf{d}^{0},Y)\cup \\ \cup TLE_{ij}((X,\textbf{c}^{0});\textbf{d}^{0},Y)\}, \\
TS_{ij}^{G}((\textbf{W}^{0},\textbf{d});Y) = \{k\in P:\;\mid\;x_{i}^{k}\wedge y_{j}^{k}\leqslant w_{ij}^{0}\vee d_{j}^{0}\}.
\end{multline}
The sets are similar to ones without ``$T$'' in denotation above (\ref{g}), however these ones use join of $x$'s with $c^{0}$'s and $d^{0}$'s instead of only $x$'s. The set $TS$ determine patterns for which weights $w^{0}$'s are grater than or equal to outputs,
since, $d_{j}^{0}\leqslant w_{ij}^{0}$ by (\ref{eq5}), (\ref{eqd0}) and the implication Definition \ref{imp} ($c\geqslant b$ in $c = a\Rightarrow b$):

\begin{multline*}
TS_{ij}^{G}((\textbf{W}^{0},\textbf{d});Y) = \{k\in P:\;\mid\;x_{i}^{k}\wedge y_{j}^{k}\leqslant w_{ij}^{0}\}.
\end{multline*}

\noindent Then, let us define the set
\begin{equation}\label{defM}
M^{wcd} = \{(\textbf{W},\textbf{c},\textbf{d})\;|\; \forall k\in P,\;((\textbf{x}^{k}\vee \textbf{c})\circ \textbf{W})\vee \textbf{d} = \textbf{y}^{k}\}.
\end{equation}
This is the set of the pattern pair family, thresholds, and the connection matrix, which satisfy the equation (\ref{eq3}). The following Theorem shows that weights $w$'s and thresholds $d$'s which give equation (\ref{eq3}) solutions, are bounded from above by the values of these $w^{0}$'s and $d^{0}$'s.
\newtheorem{Th}{Theorem}

\begin{Th}\label{th1}
Let $\textbf{W} = (w_{ij})_{n\times m}\in\mu_{n\times m},\; \textbf{c} = (c_{11}, ..., c_{nm})\in L^{n\times m},\; \textbf{d} = (d_{1}, ..., d_{m})\in L^{m},\; (\textbf{W},\textbf{c},\textbf{d})\in M^{wcd}.$ Then $\forall i\in N,\; j\in M,\; w_{ij}\leqslant w_{ij}^{0},\; d_{j}\leqslant d_{j}^{0}$. 
\end{Th}
\begin{proof}
Since $(\textbf{W},\textbf{c},\textbf{d})\in M^{wcd}$, we denote in (\ref{eq4}) $(\textbf{x}\vee \textbf{c})\circ \textbf{W} = \textbf{U}$ and get $\textbf{U}\vee \textbf{d} = \textbf{y}$. Taking the meet by $k$ of this equation, we get $\bigwedge_{k\in P}(U^{k})\vee d_{j} = \bigwedge_{k\in P}(y_{j}^{k}) = d_{j}^{0}$, since $d_{j}$'s do not depend on $k$. Thus, $d_{j} \leqslant d_{j}^{0}$.

Hence next, we proof $w_{ij}\leqslant w_{ij}^{0}$. Clear, that if $b > c$, then $a\Rightarrow b \geqslant a\Rightarrow c$. Thus, $w_{ij}^{0} = \bigwedge_{k\in P}(x_{i}^{k}\Rightarrow y_{j}^{k}) = \bigwedge_{k\in P}(x_{i}^{k}\Rightarrow (\bigvee_{l}\{(x_{l}^{k}\vee c_{lj})\wedge w_{lj}\}\vee d_{j})\geqslant \bigwedge_{k\in P}(x_{i}^{k}\Rightarrow (x_{i}^{k}\wedge w_{ij})) = \bigwedge_{k\in P}(x_{i}^{k}\Rightarrow w_{ij}) \geqslant w_{ij}$. Therefore, $w_{ij}\leqslant w_{ij}^{0}$.
\end{proof}

The following Theorem continues the previous one and shows that if there are a solution of (\ref{eq3}) with thresholds $c_{ij}$ bounded from above by $c_{ij}^{0}$, then $w^{0}$'s, $c^{0}$'s, and $d^{0}$'s provide also a solution of (\ref{eq3}).
\begin{Th}\label{th2}
Let $(\textbf{W},\textbf{c},\textbf{d})\in M^{wcd}$, i.e., $M^{wcd} \neq \emptyset$, and $c_{ij}\leqslant c_{ij}^{0}$. Then $(\textbf{W}^{0},\textbf{c}^{0},\textbf{d}^{0})\in M^{wcd}$.
\end{Th}
\begin{proof}
Let us $\forall i\in N,\;\forall j\in M,\;\forall k\in P,\; c_{ij}\leqslant c_{ij}^{0}$.
Therefore, by (\ref{eq3}), (\ref{eqd0}), (\ref{eqc0}), and Theorem \ref{th1}, we get
\begin{equation}\label{eqy}
(x_{i}^{k}\vee c_{ij}\vee d_{j})\wedge (w_{ij}\vee d_{j})
\leqslant (x_{i}^{k}\vee c_{ij}^{0}\vee d_{j}^{0})\wedge (w_{ij}^{0}\vee d_{j}^{0})\leqslant y_{j}^{k}.
\end{equation}

This is true because $w_{ij}^{0}\geqslant d_{j}^{0}\geqslant d_{j}$ by (\ref{eq5}), (\ref{eqd0}), Theorem \ref{th1} and the implication Definition \ref{imp}, $w_{ij} \leqslant w_{ij}^{0}$ by Theorem \ref{th1} and, hence, $(c_{ij}\vee d_{j})\wedge w_{ij}\leqslant (c_{ij}^{0}\vee d_{j}^{0})\wedge w_{ij}^{0} \leqslant y_{j}^{k}$ by (\ref{eqd0}), (\ref{eqc0}); also, $x_{i}^{k}\wedge w_{ij}^{0}\leqslant y_{j}^{k}$, since $x_{i}^{k}\wedge w_{ij}^{0} = x_{i}^{k}\wedge\bigwedge_{k\in P}(x_{i}^{k}\Rightarrow y_{j}^{k})\leqslant x_{i}^{k}\wedge (x_{i}^{k}\Rightarrow y_{j}^{k})\leqslant y_{j}^{k}$ by (\ref{eq5}) and the implication Definition \ref{imp}.


Thus, by (\ref{eq3}), we get:
\begin{multline*}y_{j}^{k} = \bigvee_{i\in N}\{(x_{i}^{k}\vee c_{ij}\vee d_{j})\wedge (w_{ij}\vee d_{j})\}\leqslant \\
\leqslant\bigvee_{i\in N}\{(x_{i}^{k}\vee c_{ij}^{0}\vee d_{j}^{0})\wedge (w_{ij}^{0}\vee d_{j}^{0})\}\leqslant y_{j}^{k}.
\end{multline*}
Hence, $\bigvee_{i\in N}\{(x_{i}^{k}\vee c_{ij}^{0}\vee d_{j}^{0})\wedge (w_{ij}^{0}\vee d_{j}^{0})\}= y_{j}^{k}$ for $\forall j\in M,\; k\in P$, i.e., $(\textbf{W}^{0},\textbf{c}^{0},\textbf{d}^{0})\in M^{wcd}$.
\end{proof}
The following Theorem provides us with a sufficient condition for the existence of a solution of (\ref{eq3}).
\begin{Th}\label{th3}
The set $M^{wcd} \neq \emptyset$ and $(\textbf{W}^{0},\textbf{c},\textbf{d})\in M^{wcd}$, if
\;$\forall j\in M,\; i\in N,\; c_{j}\leqslant d_{j}^{0} = \bigwedge_{k\in P}y_{j}^{k},\; \;\bigcup_{i\in N}TS_{ij}^{G}((\textbf{W}^{0},\textbf{d});Y) = P\;  and\;\forall k\in P\; \bigvee_{i:\;k\in TS_{ij}^{G}}x_{i}^{k} \geqslant y_{j}^{k}$.
\end{Th}
\begin{proof}
Let us $\forall j\in M,\;k\in P,\; \exists n\in N:\; k\in TS_{nj}^{G}((\textbf{W}^{0},\textbf{d});Y)$.
Then,
\begin{equation}
w_{nj}^{0}\geqslant x_{n}^{k}\wedge y_{j}^{k},
\label{eq7n}
\end{equation}
and
\begin{equation}
\bigvee_{n:\;k\in TS_{nj}^{G}}x_{n}^{k}\wedge y_{j}^{k} = y_{j}^{k}.
\label{eq8n}
\end{equation}

Let us consider:
\begin{multline*}
\bigvee_{i \in N}((x_{i}^{k}\vee c_{j}\vee d_{j})\wedge (w_{ij}^{0}\vee d_{j})) \geqslant \\ \geqslant \bigvee_{i \in N}(x_{i}^{k}\wedge w_{ij}^{0})\wedge y_{j}^{k} \geqslant \\ \geqslant
\bigvee_{n:\;k\in TS_{nj}^{G}}(x_{n}^{k}\wedge w_{nj}^{0}\wedge y_{j}^{k}) = y_{j}^{k},\label{eq10n}
\end{multline*}
since \ref{eq7n} and \ref{eq8n}. 

Hence,
\begin{equation}
\bigvee_{i\in N}\{(x_{i}^{k}\vee c_{j}\vee d_{j})\wedge (w_{ij}^{0}\vee d_{j})\}\in \{x_{i}^{l}\in L\; \mid\; l\in GE_{ij}(X,Y)\}.\label{eq11n}
\end{equation}
On the other hand, we get by the definitions of implication, $\textbf{d}^{0}$, and $\textbf{W}^{0}$ and Theorem 1 that $\forall i\in N,\; j\in M,\; k\in P,\,(x_{i}^{k}\vee c_{j}\vee d_{j})\wedge (w_{ij}^{0}\vee d_{j})\in \{x_{i}^{l}\in L\; \mid\; l\in LE_{ij}(X,Y)\}$, since $w_{ij}^{0}\geqslant d_{j}^{0}\geqslant d_{j}$, $x_{i}^{k}\wedge w_{ij}^{0}\leqslant y_{j}^{k}$ (since $x_{i}^{k}\wedge \bigwedge_{l}[x_{i}^{l}\Rightarrow y_{j}^{l}]\leqslant x_{i}^{k}\wedge [x_{i}^{k}\Rightarrow y_{j}^{k}]\leqslant y_{j}^{k}$), and $(c_{j}\vee d_{j}) \leqslant y_{j}^{k}$. Hence,
\begin{equation}
\bigvee_{i\in N}\{(x_{i}^{k}\vee c_{j}\vee d_{j})\wedge (w_{ij}^{0}\vee d_{j})\}\in \{x_{i}^{l}\in L\; \mid\; l\in LE_{ij}(X,Y)\}.\label{eq12n}
\end{equation}
Combining (\ref{eq11n}) and (\ref{eq12n}), we get $(\textbf{W}^{0},\textbf{c},\textbf{d})\in M^{wcd}$, i.e., $M^{wcd} \neq \emptyset$. However, if there is a set of lattice elements that satisfy the condition of the theorem: $\{(y')_{j}^{k} \mid y_{j}^{k} < (y')_{j}^{k} \leqslant w_{i_{0}j}^{0}\}$ with the same matrix $w_{ij}^{0}$, the matrix is the solution only for $y$, since $(x_{i}^{k}\vee c_{j}\vee d_{j})\wedge (w_{ij}^{0}\vee d_{j}) \leqslant y_{j}^{k}$ and this expression does not change for such $y'$\footnote{$x_{i}^{k}\wedge \bigwedge_{l}[x_{i}^{l}\Rightarrow y_{j}^{l}]\leqslant x_{i}^{k}\wedge [x_{i}^{k}\Rightarrow y_{j}^{k}] = x_{i}^{k}\wedge y_{j}^{k}$}. Hence, the proof is valid only for $y_{j}^{k}$.
\end{proof}

\section{Learning Algorithm}
\label{sec3}
In general, all $(x_{1},y_{1}), ... (x_{p}, y_{p})$ can be stored in such associative memory if  there exists a matrix--vector set $(\textbf{W}, \textbf{c}, \textbf{d})$ such as $\textbf{y} = ((\textbf{x}\vee \textbf{c})\circ \textbf{W})\vee \textbf{d}$. According
to Theorem \ref{th3}, such sets exist if the condition of
the theorem is satisfied. Therefore, to obtain solutions, it
suffices to calculate expression (\ref{eq5}) and check the
corresponding condition of Theorem \ref{th3}. However, such an algorithm does
not demonstrate the adaptability and self-regulation of multi-valued
associative memory. Therefore, we generalize the
dynamic $\delta$-learning algorithm of fuzzy associative memory,
introduced in \cite{li1}, \cite{li2}, \cite{li3} for different fuzzy cases, to the multi-valued
case as in \cite{maximov20e}. We give this algorithm here only as illustration of adaptability capabilities of such a network since it all comes down to calculating of (\ref{eq5}).

Let us note, that this algorithm fits only memories with the \emph{atomic} lattice of weights and data, since one-to-one correspondence between lattice elements and the sets of the element generator unions exist only in this case. Also, since we use linguistic variables, not numbers, and, hence, do not use activation or membership functions, we cannot use backpropagation learning methods in our learning algorithm, unlike the advanced fuzzy case (\cite{epdf}). Hence, we do not use such new effective methods for parameters updating as, e.g., in \cite{dr}. Also, let us note that we do not have such parameters here at all.

Thus, we get an \textbf{Algorithm} of the memory weights $w_{ij}$ iteration for $i\in N$ and $j\in M$. We begin from the highest lattice element $w_{ij}(0) = 1$ and go down until the next two iterations become equal.
\begin{description}
\item[Step 1.]
Initialization: for $i\in N$, $j\in M$ let us put $w_{ij}(0) = 1,\;t=0$;
\item[Step 2.]
Let $\textbf{W}(t) = (w_{ij}(t))$;
\item[Step 3.]
Let us calculate the resulting output of the memory (\ref{eq4}) $\textbf{y} = ((\textbf{x}\vee \textbf{c})\circ \textbf{W})\vee \textbf{d}$, i.e., $\forall k\in P,\;j\in M,\; y_{j}^{k}(t) = (\bigvee_{i}[(x_{i}^{k}\vee c_{ij})\wedge w_{ij}(t)])\vee d_{j}$.
Here $y_{j}^{k}$ is the stored memory output, and $y_{j}^{k}(t)$ is its iterated output which converges to  $y_{j}^{k}$ at the end of the iteration process, by (\ref{eqy}) and Theorem \ref{th5}, if a solution of (\ref{eq4}) exists (Theorem \ref{th3}).
\item[Step 4.]
Weights selection.
\newtheorem{denote}{Denotation}
\begin{denote}
From this place, we denote the union of generators of the lattice element $y_{j}^{k}$ by $\{y_{j}^{k}\}$ (this is a set) and  the matrix of such sets $\{w\}$ corresponding to the weight matrix elements by $\{w\}_{ij}$. The matrix elements are the sets of generators of the weight matrix elements.
\end{denote}
Matrices $w_{ij}$ of weights and $\{w\}_{ij}$ are one-to-one correspondent to each other in atomic lattices when every lattice element is represented by the join of its generators.
A \textbf{minus} sign will denote the \textbf{set difference}\footnote{Given set $A$ and set $B$ the set difference of set $B$ from set $A$ is the set of all element in $A$, but not in $B$.}.

Then, we choose the iterated weight with the following formulas:
\begin{equation}\label{eqw}
\{w\}_{ij}^{k}(t+1) = \left\{\begin{aligned}\{w\}_{ij}^{k}(t);\; if\; k:\;[(x_{i}^{k}\vee c_{ij})\wedge w_{ij}(t)]\vee d_{j}\leqslant y_{j}^{k};\\
\{w\}_{ij}^{k}(t) - (\{x_{i}^{k}\Rightarrow y_{j}^{k}(t)\} - \{x_{i}^{k}\Rightarrow y_{j}^{k}\}),\; otherwise.
                        \end{aligned}
                              \right.
\end{equation}
\begin{equation}\label{eqw1}
\{w\}_{ij}(t+1) = \bigcap_{k}\{w\}_{ij}^{k}(t+1).
\end{equation}
\item[Step 5.]
For $i\in N$, $j\in M$, let us check $\{w\}_{ij}(t+1) = \{w\}_{ij}(t)$? If this is
true, then the Algorithm stops, otherwise $t = t + 1$ and goes to Step 2.
\end{description}

\begin{Th}\label{th5}
Let the matrix sequence $\{\textbf{W}(t)\;|\;t = 1, 2 ...\}$ is obtained by the learning Algorithm. Then,
\begin{itemize}
\item[(a)]
$\{\textbf{W}(t)\;|\;t = 1, 2 ...\}$ is a non-increasing sequence;
\item[(b)]
$\{\textbf{W}(t)\;|\;t = 1, 2 ...\}$ converges;
\item[(c)]
$\{\textbf{W}(t)\;|\;t = 1, 2 ...\}$ converges to $\textbf{W}^{0}$, where $w_{ij}^{0}$ is defined in (\ref{eq5}).
\end{itemize}
\end{Th}
\begin{proof}
(a) For $i\in N$, $j\in M$, $k\in P$ we get from (\ref{eqw}) that $\{w\}_{ij}^{k}(t+1) \leqslant \{w\}_{ij}^{k}(t)$. Therefore, for $i\in N$, $j\in M$: $\textbf{W}(t+1)\subseteq \textbf{W}(t)$. Thus, $\{\textbf{W}(t)\;|\;t = 1, 2 ...\}$ is a non-increasing sequence.

(b) $\{\textbf{W}(t)\;|\;t = 1, 2 ...\}$ converges, since the sequence $w_{ij}(t)$ is bounded below by the smallest lattice element 0 $\forall t = 1, 2 ...$.

(c) If, all $k$'s are such that $\forall t:\; [(x_{i}^{k}\vee c_{ij})\wedge w_{ij}(t)]\vee d_{j}\leqslant y_{j}^{k}$, then $w_{ij}(t) = w_{ij}(0) = w_{ij}^{0} = 1$ by (\ref{eqw}), (\ref{eqw1}), and Theorem \ref{th1}.

In particular, if there are not (with some denotation simplification) $k\in TG_{ij}\cup NC_{ij}$, i.e., all $x_{i}^{k} \leqslant y_{j}^{k}$, then $\forall t:\; [(x_{i}^{k}\vee c_{ij}^{0})\wedge w_{ij}(t)]\vee d_{j}^{0}\leqslant y_{j}^{k}$ by (\ref{eqc0}), (\ref{eqd0}). Therefore, again, by (\ref{eqw}), (\ref{eqw1}), and (\ref{eq5}), $w_{ij}(t) = w_{ij}(0) = w_{ij}^{0} = 1$.

If such $k$'s exist that $[(x_{i}^{k}\vee c_{ij})\wedge w_{ij}(t)]\vee d_{j}\nleqslant y_{j}^{k}$, then, we get for these $k$ on the basis of the definition of the sets' difference operation:

$\{w\}_{ij}^{k}(1) = 1 - (\{x_{i}^{k}\Rightarrow \bigvee_{\substack{i\in N}}([(x_{i}^{k}\vee c_{ij})\wedge w_{ij}(0)]\vee d_{j})\} - \{x_{i}^{k}\Rightarrow y_{j}^{k}\}) = \{x_{i}^{k}\Rightarrow y_{j}^{k}\} \geqslant \{w\}_{ij}^{0}$, since the first term in parentheses is equal to 1: $\{x_{i}^{k}\Rightarrow \bigvee_{\substack{i\in N}}([(x_{i}^{k}\vee c_{ij})\wedge w_{ij}(0)]\vee d_{j})\} = 1$.

$\{w\}_{ij}^{k}(2) = \{x_{i}^{k}\Rightarrow y_{j}^{k}\} - (\{x_{i}^{k}\Rightarrow \bigvee_{\substack{i\in N}}([(x_{i}^{k}\vee c_{ij})\wedge \{w\}_{ij}(1)]\vee d_{j})\} - \{x_{i}^{k}\Rightarrow y_{j}^{k}\})$. From here we obtain again, on the basis of the definition of the sets' difference operation (the term in parentheses does not contain elements from all $\{x_{i}^{k}\Rightarrow y_{j}^{k}\}$):

$\{w\}_{ij}^{k}(2) = \{x_{i}^{k}\Rightarrow y_{j}^{k}\} = \{w\}_{ij}^{k}(1)$.
Thus, the Algorithm converges in the second step.

It converges to $\textbf{W}^{0}$, since $w_{ij}(2) = w_{ij}(1) = \bigwedge_{k\in P}(x_{i}^{k}\Rightarrow y_{j}^{k}) = w_{ij}^{0}$.

\end{proof}

Thus, the Algorithm converges to $\textbf{W}^{0}$ and gives the solution $(\textbf{W}^{0}, \textbf{c}^{0}, \textbf{d}^{0})$ of (\ref{eq4}) if such a solution exists under the conditions of Theorem \ref{th3}. However, we can say nothing about existence of the solution $(\textbf{W}^{0}, \textbf{c}, \textbf{d})$ in general.

\section{Discussion}
\label{sec4}
Theorem \ref{th1} tells us that if $\textbf{W}^{0}$ is a solution, then it provides the largest solution of the equation (\ref{eq4}), and we have used this fact in Sec. \ref{sec3} when we went down from the largest lattice element to $\textbf{W}^{0}$ in a learning algorithm. The same is true for thresholds $\textbf{d}^{0}$.

However, we can say much less about thresholds $\textbf{c}$ unlike the fuzzy case where similar $c^{0}$'s are the maximal possible ones, so $(\textbf{W}^{0}, \textbf{c}^{0}, \textbf{d}^{0})$ is the largest solution of (\ref{eq4}) \cite{epdf}. Indeed, we can rewrite (\ref{eq3}) as
\begin{equation}\label{eq3n}
y_{j}^{k} = (\bigvee_{i}\{(x_{i}^{k}\vee c_{ij})\wedge w_{ij}\})\vee d_{j} = \bigvee_{i}[(x_{i}^{k}\wedge w_{ij})\vee (c_{ij}\wedge w_{ij})]\vee d_{j}.
\end{equation}
All solutions $w_{ij}\leqslant w_{ij}^{0} = \bigwedge_{k\in P}(x_{i}^{k}\Rightarrow y_{j}^{k})\geqslant \bigwedge_{k\in P}y_{j}^{k}\leqslant c_{ij}^{0}$. Also, all $(x_{i}^{k}\wedge w_{ij}\})\leqslant y_{j}^{k}$ and $(c_{ij}\wedge w_{ij})\leqslant y_{j}^{k}$. Therefore, $\textbf{c}$ may be almost arbitrary towards $\textbf{c}^{0}$.


The meaning of  Theorem \ref{th3} is that not every
pair of $x - y$ patterns can be stored in such associative memory with confidence (as
well as in \cite{maximov20} and in a fuzzy case) -- sufficiently large meets of
implications in $x - y$ patterns in (\ref{eq5}) should be enough for $\bigcup_{i\in N}TS_{ij}^{G}((\textbf{W}^{0},\textbf{d}^{0});Y) = P$. Besides, all $x$'s in every pattern cannot concentrate in the lower zone (Fig. \ref{fig2}), otherwise combining them may not give $y$. The storing capacity of the variant with thresholds is higher than without them, since combining $\textbf{c}$ with the $x$'s in the lower zone increases them. Also, the same is true for $d$.
However, we cannot assert that $c_{ij}^{0}$'s provide the largest capacity: $c_{ij}\wedge w_{ij}^{0}$ may give additional capabilities in the case of $c_{ij}$ is incomparable with $\bigwedge_{k\in P}y_{j}^{k}$ and Theorem \ref{th3} does not hold.

Finally, we will indicate without proof (the proof in\cite{epdf} does not change in the multi-valued case) that increasing number of neuron layers does not entail an increase of capacity.

The next point concerns the computational complexity.
We have seen that implications $x_{i}^{k}\Rightarrow y_{j}^{k}$'s calculating is necessary to obtain the weight matrix in (\ref{eq5}), (\ref{eqw}). In the case of a complicated lattice, such a calculation may be an elaborative task: an atomic lattice with $N$ generators has $2^{N}$ elements. However, this problem is resolved in \cite{maximov20e}, \cite{maximov20} using a multi-valued neural network\footnote{Also, the use of a multi-valued associative memory similar to this article is proposed to obtain a control solution  quickly.}. Such a network has the lattice $L$ as a unique layer and a $\delta$-like activation function for a neuron in each lattice element. Such neurons skip exiting only those inputs, which may be the implication value. Then, their maximum gives the real implication. Hence, we have to obtain the maximum of at most $N-1$ nested sets of generators in order to calculate one implication. Thus, for $M$ outputs and $P$ patterns, we get at most $(N-1)^{2}NMP$ complexity. We have to add $O(P)$ to obtain meets of $k$ patterns in $w_{ij_{0}} = \bigwedge_{k\in P}(x_{i}^{k}\Rightarrow y_{j}^{k})$. Hence, computational complexity is linear by $M$  and cubic by $N$ and $P$.

There is no need to talk about cost effectiveness in our task, since we obtain an exact solution of equation (\ref{eq4}) by the Algorithm and we do not select network parameters (we do not have them at all in the associative memory) as in, e.g., \cite{dr}. Hence, we can store in the memory all training patterns with confidence, if Theorem \ref{th3} holds.

\section{Simulation Example}
\label{sec5}
We will use a complete Brouwerian lattice $L$ as a set of values of neural network variables in the processing of control and diagnostic information of complex dynamic objects, e.g., aircraft and robotic complexes, as expanding the regular neural networks approach of \\ \cite{gonch},\\ \cite{gonch1}.

These authors have suggested an algorithm for functional diagnostics of such complexes, based on neural network technologies. We expand this approach to classifying spacecraft trajectories with linguistic lattice-valued estimations of trajectory parameters. Our classifying problem differs from, e.g., system state estimating in \cite{er} or big data classification in \cite{fr} since we consider rather limited sets of linguistic variables which can be observed by humans (experts). Also again, we have no deal with numbers, thus our process representation is quite different. However, it seems that the equations of different estimators in \cite{er} may be considered with multi-valued values similarly to Maximov D. ``Multi-valued Cognitive Maps''.

Thereby, let us consider a system where the state is characterized by four linguistic variables --- $T_{1}...T_{4}$ --- which take values in correspondent subsets of the lattice $L$ Fig. \ref{fig3}.
\begin{figure}\begin{center}
  \includegraphics[scale=1.0]{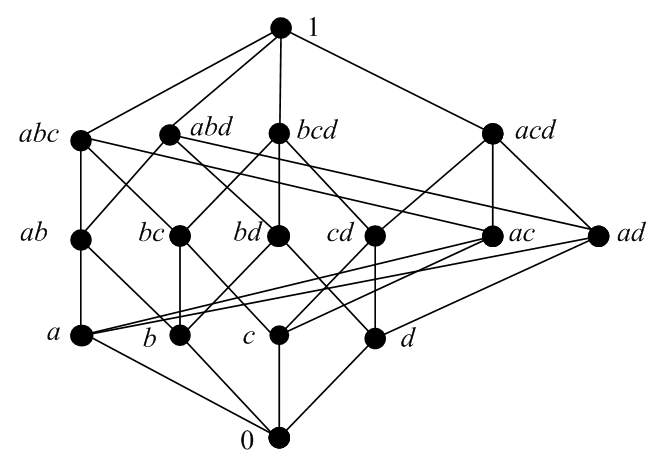}
\end{center}
\caption{A lattice for control and diagnostic information estimating}
\label{fig3}       
\end{figure}
Our task is to decide which of the two classes --- $C_{1}$ or $C_{2}$ --- relates to the current distribution of estimates of these variables. We can imagine these variables as an alphabet of the simplest aircraft trajectories, e.g., a one-sided manoeuvre, snake manoeuvre, spiral manoeuvre, etc., and classes $C_{1}$ and $C_{2}$ as a directional movement type or a chaotic one.

Then, the lattice estimations of $T_{i},\;C_{j}$ denote degrees of confidence/truth of the correspondent valuation. However, these degrees of confidence are also, from another side, some terms in correspondent linguistic variables, e.g., the set $\{a, b, c, d\}$ may be considered as the set of directions $\{up, down, right, left\}$ diagnosed by linguistic sensors, i.e., sensors providing linguistic assessments of directions. Then, joins $ab, ...$ are the moves with such possible directions' uniting or alternating compositions. Hence, we evaluate the linguistic variables also through linguistic terms.

In the approach, the variable $T_{1}$ --- e.g., horizontal snake manoeuvre --- may be estimated by terms $\{ab, abc, abd\}$, the variable $T_{4}$ --- e.g., vertical snake manoeuvre --- may be estimated by terms $\{cd, acd, bcd\}$, the variable $T_{2}$ --- spiral manoeuvre --- by terms $\{ac, ad, bc, bd\}$, the variable $T_{3}$ --- one-sided manoeuvre --- by $\{a, b, c, d\}$, etc. The terms are really the lattice $L$ correspondent sublattices, and the sublattices' elements evaluate them. Elements of the whole lattice $L$ estimate the classes $C_{1}$ of directional moves, and $C_{2}$ of chaotic moves.

Thus, we consider the neural network depicted in Fig. \ref{fig4}, in which the input variables $x_{i}$ are some valuations of the terms estimating variables $T_{i}$.
\begin{figure}\begin{center}
  \includegraphics[scale=0.7]{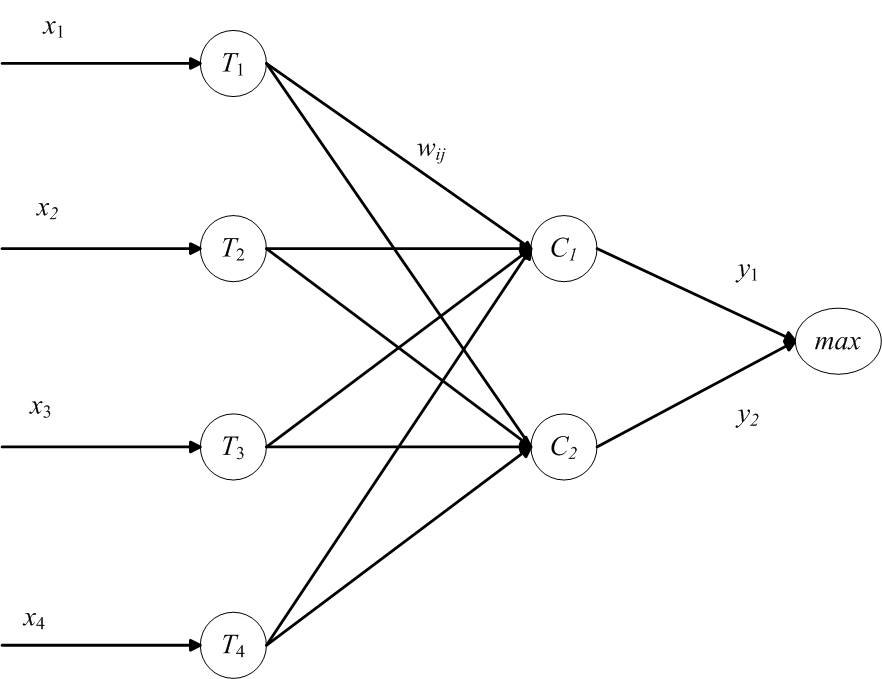}
\end{center}
\caption{The neural network topology for the system moves estimations}
\label{fig4}       
\end{figure}
We must be able to decide which class $C_{i}$ the given object (e.g., a trajectory) belongs to if the input variables $x_{k}$ are estimated by the quantities from $T_{j}$ terms.

Let us consider the next learning pattern pair family $\{(\textbf{x}^{k}, \textbf{y}^{k})\;|\;k\in P\}$ (Table \ref{tb}).
\begin{table}
\begin{center}
$\begin{array}{c|c|c}
k & x & y \\ \hline
1 & (ab, a, a, cd) & (ac, 1) \\ \hline
2 & (bc, bc, c, bcd) & (bcd, bc) \\ \hline
3 & (abc, c, c, c) & (abc, c) \\ \hline
4 & (b, bd, 0, bc) & (1, bc) \\ \hline
5 & (a, ac, a, acd) & (acd, cd) \\ \hline
6 & (ac, ac, b, bc) & (c, abc) \\ \hline
7 & (c, bd, d, d) & (bc, cd) \\ \hline
8 & (d, ad, d, ac) & (cd, acd)
\end{array}$
\end{center}
\caption{Learning pattern pair family}
\label{tb}
\end{table}
Though, our lattice interpretation assumes symmetry by an $a, b, c, d$ situation, we picked up our example to demonstrate threshold using. Thus, we obtain $\textbf{c}^{0} = ((c,c); (c,c); (c,c); (c,c)),\;\textbf{d}^{0} = (c,c)$ and the connection weights matrix:
 $${\textbf{W}^{0}}^{T} = \begin{pmatrix}
 cd,& cb,& ac,& c \\
 cd,& c,& bcd,& bc
 \end{pmatrix}.$$
Such obtained thresholds mean that our training family assumes the direction to the right as a priori priority: all inputs rise up to include it.

We may easily show that the given pattern pair family $\{(\textbf{x}^{k}, \textbf{y}^{k}) | k\in P\}$ satisfies the conditions that $\forall j \in M,\; \bigcup_{i\in N}TS_{ij}^{G}((\textbf{W}^{0}, \textbf{d}^{0});Y) = P$. Hence, by Theorem \ref{th3}, all pattern pairs in Table \ref{tb} can be stored in such a multi-valued associative memory.

We may also investigate classifying properties of this neural network. Let us consider next the possible input patterns with the correspondent classification obtained (greatest output truth values determine the classification):
\begin{center}
$\begin{array}{c|c|c|c}
k & x  & y & C_{i}\\ \hline
9 & (abd, bd, d, d)  & (bcd, cd) & C_{1} \\ \hline
10 & (ad, 0, b, ad)  & (cd, bcd) & C_{2} \\ \hline
11 & (bd, bd, b, a)  & (bcd, bcd) & C_{1} \;or\; C_{2}
\end{array}$
\end{center}
These patterns were selected, as well as learning ones, based on our interpretation of lattice elements as aircraft manoeuvres. Thus, we see that the neural network classifies patterns quite reasonably. However, it should be noted that the system is susceptible to variations in input patterns, and it is not always possible to get such a meaningful interpretation.

In \cite{gonch1}, the average probability of recognizing the type of maneuver being performed based on noisy measuring information is 0.78. In our case, noise is not essential due to roughness of linguistic assessments. However, it is difficult to estimate the recognizing possibilities of multi-valued classifier since there are currently no real data associated with the multi-value patterns. We can only establish that al training patterns are recognized and that other patterns may take incomparable estimations. Hence, we may make a definite decision not always.

\section{Conclusion}
\label{sec6}
This study introduces a new concept of multi-valued neural networks with thresholds, in which weights and data represent not by numbers. The idea continues our studies in which the state of a system is estimated not by numbers, but by elements of a partially ordered set, namely, the lattice. In our case, the lattice is finite and distributive, and it consists of some linguistic variables. Such an approach can facilitate a situation which requires the participation of expert evaluation: in this case --- assessments are the linguistic variables, not numbers. Thus, they do not require vague concepts of defuzzification.

We have expanded the results on fuzzy neural networks to such a multi-valued case in which variables take value in a Brouwer lattice. We discovered conditions under which it is possible to store given pairs of linguistic patterns in such a multi-valued associative memory. Here, thresholds increase the network storing capacity as in the fuzzy case. We also gave the learning
algorithm generalizing the fuzzy one without thresholds. Finally, we gave the example of the network use to classifying aircraft/spacecraft trajectories'.

However, the classification method has an obvious limitation arising from the nature of the representation of variable: our classifier does not always distinguish the system states since they may be incomparable (as the lattice elements). Also, we were not able to get the solutions with the highest capacity. We have obtained only the condition of existence of the solution with largest weights and certain thresholds.

Future researches are going in two directions. First, we develop multi-valued cognitive maps which have the neural network nature. Second, we suppose to expand use of a lattice elements as weights and data values to the lattice subsets, as analogue of fuzzy numbers. We assume such an expansion will give an opportunity for the networks use in more complicated cases, e.g., in image restoration.

\section*{Conflicts of interest}
%

The authors declare that they have no conflict of interest.



\bibliographystyle{spbasic}      

  \bibliography{Maximov_Multi_Valued_r}





\end{document}